\documentclass[11pt]{article}
\usepackage{lmodern}
\usepackage[T1]{fontenc}
\usepackage[latin9]{inputenc}
\usepackage{geometry}
\usepackage{color}
\usepackage{float}
\usepackage{subcaption} 
\usepackage{amsmath,amssymb,amsfonts,amsthm}
\usepackage{algorithmic}
\usepackage{graphicx}
\usepackage{textcomp}
\usepackage{xcolor,colortbl}
\usepackage{amsthm}
\usepackage{subcaption} 
\usepackage{algorithm}
\usepackage{booktabs} 
\usepackage{url}
\usepackage{multicol}
\usepackage[normalem]{ulem}

\usepackage{lipsum}
\usepackage{bm}
\usepackage[show]{chato-notes}
\usepackage{stfloats}
\usepackage{balance}

\newcommand{\setcover}{{\sc SetCover}}

\newcommand{\balanceloadcover}{{\sc BalancedTA}}
\newcommand{\balanceloadrequiredcover}{{\sc R-BalancedTA}}

\newcommand{\load}{{\sc balanced task covering}}

\newcommand{\ExpertGreedy}{{\texttt{ExpertGreedy}}}
\newcommand{\ProjectTopExpertGreedy}{{\texttt{TaskGreedy}}}

\newcommand{\SetCoverGreedy}{{\texttt{SetCover}}}
\newcommand{\BestCostGreedy}{{\texttt{BestCostGreedy}}}
\newcommand{\PairGreedy}{{\texttt{PairGreedy}}}

\newcommand{\LoadGreedy}{{\texttt{Load}}}
\newcommand{\BestLoad}{{\texttt{BestLoad}}}

\newcommand{\TopTasks}{{\texttt{TopTasks}}}
\newcommand{\TopExperts}{{\texttt{TopExperts}}}
\newcommand{\UpdateTeams}{{\texttt{UpdateTeams}}}
\newcommand{\UpdateTasks}{{\texttt{UpdateTasks}}}
\newcommand{\UpdateExperts}{{\texttt{UpdateExperts}}}

\newcommand{\cardinality}{\ensuremath{N}}

\newcommand{\GuruDataset}{{\textit{Guru}}}

\newcommand{\FreelanceDataset}{{\textit{Freelancer}}}
\newcommand{\UpworkDataset}{{\textit{Upwork}}}

\providecommand{\definitionname}{Definition}
\providecommand{\lemmaname}{Lemma}
\providecommand{\theoremname}{Theorem}
\newtheorem{problem}{Problem}

\newtheorem{thm}{\protect\theoremname}[section]

\newtheorem{lem}[thm]{\protect\lemmaname}
 \newtheorem{observation}{Observation}

\newcommand{\spara}[1]{\smallbreak\noindent{\bf{#1}}}
\newcommand{\mpara}[1]{\medskip\noindent{\bf{#1}}}

\newcommand{\etal}{{et al.}}

\newcommand*{\belowrulesepcolor}[1]{%
	\noalign{%
		\kern-\belowrulesep
		\begingroup
		\color{#1}%
		\hrule height\belowrulesep
		\endgroup
	}%
}
\newcommand*{\aboverulesepcolor}[1]{%
	\noalign{%
		\begingroup
		\color{#1}%
		\hrule height\aboverulesep
		\endgroup
		\kern-\aboverulesep
	}%
}

\newcommand{\squishlist}{\begin{list}{$\bullet$}
  { \setlength{\itemsep}{0pt}
     \setlength{\parsep}{3pt}
     \setlength{\topsep}{3pt}
     \setlength{\partopsep}{0pt}
     \setlength{\leftmargin}{1.5em}
     \setlength{\labelwidth}{1em}
     \setlength{\labelsep}{0.5em} } }
\newcommand{\squishend}{
  \end{list}  }

\begin{document}

\title{Finding teams that balance expert load and task coverage \thanks{Department of Computer Science, Boston University. \{smnikol,marcocai,evimaria\}@bu.edu}}
\author{Sofia Maria Nikolakaki \and Mingxiang Cai \and Evimaria Terzi}
\date{}

\maketitle

\begin{abstract}
The rise of online labor markets (e.g., Freelancer, Guru and Upwork) has ignited a lot of research on team formation, where experts acquiring different skills form teams to complete tasks.
The core idea in this line of work has been the strict requirement that the team of experts assigned to complete a given task should contain a superset of the skills required by the task.
However, in many applications the required skills are often a wishlist of the entity that posts the task and not all of the skills are absolutely necessary.
Thus, in our setting we relax the complete coverage requirement and we allow for tasks to be partially covered by the formed teams, assuming that the quality of task completion is proportional to the fraction of covered skills per task. 
At the same time, we assume that when multiple tasks need to be performed, the less the load of an expert the better the performance. 
We combine these two high-level objectives into one and define the {\balanceloadcover} problem. 
We also consider a generalization of this problem where each task consists of \emph{required} and \emph{optional} skills. In this setting, our objective is the same under the constraint that all required skills should be covered.
From the technical point of view, we show that the {\balanceloadcover} problem (and its variant) is NP-hard and design efficient heuristics for solving it in practice. 
Using real datasets from three online market places, Freelancer, Guru and Upwork  we demonstrate the efficiency of our methods and the practical utility of our framework.
\end{abstract}

\section{Introduction}
\label{sec:intro}
Creating effective teams by combining experts with diverse skills allows organizations to successfully complete tasks from different domains, while also helping individual experts position themselves in highly competitive job markets.
The rise of online labor markets, like Freelancer, Guru and Upwork has motivated a significant amount of research on the problem of \emph{team formation}~\cite{anagnostopoulos2012online,bhowmik2014submodularity,kargar2011discovering,kargar2012efficient,lappas2009finding,majumder2012capacitated,rangapuram2013towards}.
Although many variants of this problem have been considered, they all rely on the same model of \emph{coverage} of tasks by experts. 
According to this model, both tasks and experts are characterized by a set of skills. 
Each task consists of a set of skills required for the task completion and each expert acquires a set of skills. 
A team consisting of one or more experts completes a task if the union of the skills of the team's experts covers all the required skills of the task.  In all existing work
\emph{complete coverage} of tasks is required by the formed teams.

However, there are cases where tasks are described generically and require teams to have strong background 
in many areas, whereas clearly not all these areas can (or need to) be simultaneously covered.
Oftentimes, the list of required skills is more like the wishful thinking of the entity that posts 
the task and not all of the skills need to be covered for its completion.
For example, think of a post that describes potential cluster hires for an academic institution.
Moreover, when looking at job posts in online labor markets, oftentimes the skills required by the posted tasks are repetitive. 
An example of a job post in \url{guru.com} is: \textit{basic, oracle, html, java, javascript, mysql, css, sql, http, ajax, mvc, architecture, jquery, software, software development, web development, developer, web developer}. 
Another example from \url{freelancer.com} is: \textit{Advertising, Facebook Marketing, Internet Marketing, Marketing, Social Networking}. 
In these examples the task descriptions could be repetitive; someone who is good at  Marketing is also good at Facebook and/or Internet marketing. 
Therefore, in such cases not all skills of a task need to be covered.

Motivated by these settings, we relax the requirement of completely covering a task and we allow for \emph{partial task coverage} by the formed teams.
However, we assume that the quality of task completion is proportional to the fraction of covered skills.

Another important factor for the quality of task completion is the load of the experts.
For example, an expert assigned to several tasks may be too busy to devote a lot of effort
in each one of them and consequently underperform~\cite{anagnostopoulos2010power}. 
Thus, our second objective is to avoid loading the experts with many tasks.

We combine these two objectives as follows: given a set of tasks, form one team
per task such that the tasks are partially covered, while at the same time, the maximum number of teams an expert participates in is kept as small as possible. 
Thus, our goal is to both \emph{fulfill the tasks as much as possible}
and \emph{not overload the experts} involved with the formed teams.

Formally,  our goal is the following: given a set of $k$ tasks $\mathcal{J}$ and a set of experts $\mathcal{P}$
form $k$ teams $\mathcal{Q}=\{Q_1,\ldots,Q_k\}$, one for each task such that
$
\lambda L(\mathcal{Q})+ C(\mathcal{Q},\mathcal{J})
$
is minimized.  
In this equation, $L(\mathcal{Q})$ denotes the maximum number of teams an expert participates
in and $C(\mathcal{Q},\mathcal{J})$ denotes the sum of the fraction of uncovered skills per task. Both these components are quantities that we 
aim to minimize, while $\lambda$ is a trade-off parameter that controls the importance of each objective.

We call the above problem the {\balanceloadcover} problem.
Note that the problem definition is such that combines the two, seemingly unrelated objectives, into a single objective, and  
asks from the algorithm to find the right balance between the two, according to the trade-off parameter $\lambda$, and without placing a hard constraint on one of the objectives.

For the {\balanceloadcover} problem we show that
there are values of $\lambda$ such that the {\balanceloadcover} problem is NP-hard
and  we design algorithms for solving it efficiently in practice.
The effectiveness and the efficiency of our algorithms is shown in our experimental 
evaluation with datasets from three major online labor markets.

We also note that in many applications, there may be tasks that have both \emph{absolutely required} and \emph{optional} skills.  
For instance, a task might require workers with expertise in \textit{Facebook Marketing and Advertising} and optionally broader knowledge in the areas of \textit{Internet Marketing, Marketing and Social Networking}.
Thus, we also consider a generalization of the {\balanceloadcover}, where there is a hard 
constraint on the coverage of the required skills for every 
task. 
Not only do we show that this variant of the problem is also NP-hard, but also that the
same algorithms we designed for {\balanceloadcover} can be used for this version
of the problem as well, with minimal changes.

Our contributions are summarized as follows:

\squishlist
\item We define the {\balanceloadcover} problem, which tries to find teams for a set of tasks such that both
the coverage of task requirements (in terms of skills) is maximized and the load of every individual worker (in terms of the number of teams she participates in) is minimized. We combine these two requirements in a single objective and use a trade-off parameter $\lambda$ to control the importance of each.
\item We study the computational complexity of the {\balanceloadcover} problem and show that while there are cases for 
which it is trivial, there are also cases for which it is NP-hard.
\item We design a set of heuristics for solving {\balanceloadcover} in practice.
\item We study a variant of the {\balanceloadcover} problem where the set of skills required for each task is split into required and optional.  
We show that this version is NP-hard and also demonstrate that our algorithms for {\balanceloadcover} can be applied to it with few modifications.
\item Finally, using three datasets from real online labor markets, we test the practical utility and the efficiency of our methods.
\squishend

\mpara{Roadmap:} Section~\ref{sec:related} reviews the related work.
We present our problem in Section~\ref{sec:prelim} and our algorithms in Section~\ref{sec:algo}. In Section~\ref{sec:exp} we evaluate the performance of our methodology using real datasets. We conclude the
paper in Section~\ref{sec:concl}.

\section{Related Work}
\label{sec:related}
Recent studies raise the importance of team formation in different settings \cite{li2018network,wang2015comparative}. 
To the best of our knowledge, we are the first to introduce the {\balanceloadcover} problem, where we 
simultaneously optimize for the coverage of the task requirements and the load of the experts. 
However, our work is related to existing work on team formation as described below.

\mpara{Team formation in network of experts:} 
Lappas {\etal} \cite{lappas2009finding} were the first to introduce the notion of team formation in the setting of a social network. Given a network of experts with skills, their goal is to find a team that collectively covers 
all the requirements of a single task, while establishing small communication cost (in terms of the network) between the team members. 
A series of subsequent works extended this work towards different directions~\cite{anagnostopoulos2012online,bhowmik2014submodularity,kargar2013finding,kargar2011discovering,kargar2012efficient,majumder2012capacitated,li2015team,li2015replacing,li2017enhancing,rangapuram2013towards,yin2018social} 
All the aforementioned works share two common assumptions: 
$(i)$ the experts are organized in a network that quantifies how well they can work together and $(ii)$ all the required
skills of the tasks need to be covered by the formed teams.
Our model does not assume the existence of a network among the experts 
and the tasks need not be fully covered.
Therefore, the computational problem that we are solving is different from the ones above.

\mpara{Team formation with load balancing:} 
Anagnostopoulos {\etal}~\cite{anagnostopoulos2010power}, 
were the first to consider minimizing the load of experts in the online setting where a stream of tasks arrives and
experts form teams in order to cover all the required skills for each task. The offline version of this problem
resembles the load-balancing requirement of our problem. However,  our work allows partial task coverage, while Anagnostopoulos {\etal} \cite{anagnostopoulos2010power} form teams that entirely cover the requirements of a task. Moreover, our framework also provides the flexibility of defining a desirable trade-off between the two costs and creates effective teams based on the importance of each.

\mpara{Multiple tasks coverage:} A key characteristic of our work is that we consider the offline setting where 
there are multiple tasks, known a-priori, and there is a team formed for each one of them. 
The offline versions of Anagnostopoulos {\etal}~\cite{anagnostopoulos2010power, anagnostopoulos2012online}, the work of Golshan {\etal}~\cite{golshan14profit}, as well as the recent work of Barnab\'o {\etal}~\cite{barnabo2019algorithms} consider multiple tasks and multiple teams. However,
contrary to our setting, all the above works require that all skills required by the sequence of tasks are completely covered. 

\mpara{Team formation with partial coverage:} 
Probably, the closest to our work is the work by Dorn and Dustdar~\cite{dorn2010composing}, which introduces a multi-objective team composition problem with two objectives: skill coverage and communication cost. Their goal is to identify the best balance between the two costs. For this purpose, they use a set of heuristics that self-adjust a trade-off parameter to decide team configurations. In our setting, we do not consider the communication cost, but the workload of the experts. Moreover, our algorithms focus on allocating experts to teams, based on a user-defined trade-off between load and coverage. Dorn and Dustar focus on \emph{finding} a ``best" trade-off between connectivity and coverage, where the notion of ``best" is defined in a rather adhoc manner.
Finally, although they touch upon the issue of partial coverage, they focus on data extraction rather than algorithm design.

\section{Preliminaries}
\label{sec:prelim}
This section provides the notation used throughout the paper and presents the formal definition and the
complexity analysis of the problem that we study. 

Throughout the discussion, we consider a set of $m$ skills $\mathcal{S}$, a set of $n$ experts $\mathcal{P} = \{P_{i}; i=1,\ldots,n\}$ and a set of $k$ tasks $\mathcal{J} = \{J_{j}; j=1,\ldots,k\}$. 
In this setting, every expert and every task is a subset of the skills, i.e., $P_{i}\subseteq \mathcal{S}$ and $J_{j}\subseteq \mathcal{S}$, respectively. 
To complete a task we need to assign a team of experts to it.
We let $Q_j\subseteq \mathcal{P}$ denote the team assigned to the $j$th task. 
For $k$ tasks, we form $k$ teams $\mathcal{Q} = \{Q_1,\ldots,Q_k\}$. 
We call $\mathcal{Q}$ the \emph{team assignment} for tasks $\mathcal{J}$.
For each team $Q_j$ we compute its \emph{skill profile} $\textsc{Cov}(Q_j)$ representing the union of the skills of its members.
That is, $\textsc{Cov}(Q_j) = \cup_{i\in Q_j}P_i$.
 
\mpara{Load cost ($L$)}: An important quantity is the \textit{load} of a person, which is the number of tasks a person is assigned to. That is, for person $P$
we have that the load of $P$ is $L(P,\mathcal{Q})$ = $\left|\{j:P\in Q_{j}\}\right|$. 
We are interested in the \emph{maximum load among all experts}, i.e., 
\begin{equation*}
	\begin{aligned}
	& L(\mathcal{Q})=\max\limits_{P\in \mathcal{P}} L(P,\mathcal{Q}) \\
	\end{aligned}
	\label{eq:phs_orig}
\end{equation*}

\mpara{Incompleteness cost ($C$)}: 
Given a task $J$ and a team $Q\subseteq \mathcal{P}$ assigned to it we define the \emph{incompleteness cost} of $Q$ with respect to $J$ to be the fraction of the required skills, which are \emph{not covered} by the team's skill profile.
That is, 

\begin{equation*}
	\begin{aligned}
	& F(Q,J) = \frac{\left|J\setminus \textsc{Cov}(Q)\right|}{|J|}
	\end{aligned}
    \label{eq:prob}
\end{equation*}

Intuitively, our goal is to minimize the incompleteness cost since we want the assigned team to cover as many of the skills required by the corresponding task as possible.
Thus, we define the total \emph{incompleteness cost} of a team assignment $\mathcal{Q}$ to be:
\begin{equation*}
	\begin{aligned}
	& C(\mathcal{Q},\mathcal{J}) = \sum_{j\in \mathcal{J}} F(Q_{j},J_{j})
	\end{aligned}
    \label{eq:prob1}
\end{equation*}

\mpara{Team-assignment cost ($B$)}: 
Given a trade-off parameter $\lambda$, the cost of a team assignment 
$\mathcal{Q}$ for a set of tasks $\mathcal{J}$ is denoted by
$B(\mathcal{Q},\mathcal{J},\lambda)$ and it is a linear combination of the maximum workload and the incompleteness cost as defined above.
That is,
\begin{equation*}
	\begin{aligned}
	& B(\mathcal{Q},\mathcal{J},\lambda) = \lambda L(\mathcal{Q}) + C(\mathcal{Q},\mathcal{J})
	\end{aligned}
	\label{eq:prob2}
\end{equation*}
The trade-off parameter $\lambda$ provides an easy way to control the relative importance of each objective, where $\lambda = 0$ ignores the workload, and conversely $\lambda > k$, where $k$ is the number of tasks, ignores the incompleteness cost. In our experiments, we
considered different values of $\lambda\in\mathbb{R}^+$ and we discuss our findings.

\subsection{The {\balanceloadcover} problem}
We can now define the main problem addressed in this paper:
\begin{problem}[{\balanceloadcover}]
	\label{prob:blc}
	Given a set of $k$ tasks $\mathcal{J}=\{J_1,\ldots, J_k\}$, a set of $n$ experts $\mathcal{P}=\{P_1,\ldots,P_n\}$ and a real non-negative  
	value $\lambda\in \mathbb{R}^+$, find a team assignment
$\mathcal{Q}=\{Q_1,\ldots , Q_k\}$ consisting of
$k$ teams, such that team $Q_i$ is associated with task $J_i$ and
$B(\mathcal{Q},\mathcal{J},\lambda)$ is minimized. 
\end{problem}
The solution of the case where $\lambda=0$ is trivial: when $\lambda=0$ then the optimal solution assigns all workers to all tasks, leaving this way the minimum number of non-covered skills. 

On the other hand, when $\lambda > k$, where $k$ is the number of tasks, we prove that only the workload matters and thus the trivial solution of not assigning experts to tasks is the best strategy.  This is summarized in the following lemma. 
\begin{lem}
\label{lemma:load}
For a set of $k$ tasks and $\lambda > k$, the optimal solution of the {\balanceloadcover} problem is the one that leaves all tasks completely uncovered; i.e., $\mathcal{Q}=\emptyset$.
\end{lem}

\begin{proof}
Assume for the sake of contradiction an optimal solution $\mathcal{Q}^{*}\neq \emptyset$ with corresponding workload $L(\mathcal{Q}^{*}) \geq 1$.
By the definition of incompleteness we know that $0 \leq C(\mathcal{Q}^{*}) \leq k$ and therefore $B(\mathcal{Q}^{*},\mathcal{J},\lambda) > k$, for $\lambda > k$.
However, we see that there exists a solution $\mathcal{Q}^{'}$ with corresponding workload $L(\mathcal{Q}^{'})=0$ whose team-assignment cost is exactly equal to $k$, i.e., $B(\mathcal{Q}^{'},\mathcal{J},\lambda) = k$.
By the definition of load this solution can only be $\mathcal{Q}^{'}=\emptyset$ which contradicts the initial assumption.
\end{proof}

Therefore, we consider problem instances where $\lambda$ takes values in the
range $(0,k]$, where $k$ is the number of tasks.  For a subset of these values of $\lambda$
we can prove that the {\balanceloadcover} problem is NP-hard. More specifically, we have the following complexity result.

\begin{thm}
\label{thm:np}
For a set of $k$ tasks the {\balanceloadcover} problem is NP-hard for $0 < \lambda < \frac{1}{k\cardinality}$, with $N$ being the cardinality of the largest task.
\end{thm}

\begin{proof}
For the rest of the proof, we will refer to the {\balanceloadcover} problem for the case $0 < \lambda < \frac{1}{k\cardinality}$.
We reduce an instance of the NP-hard {\load} problem~\cite{anagnostopoulos2010power} to the {\balanceloadcover} problem. 
A reduction from {\load} to {\balanceloadcover} exists, if and only if a solution instance of {\balanceloadcover} for $0 < \lambda < \frac{1}{k\cardinality}$ is also a solution to {\load}.

An instance of {\load} consists of a pool of experts and tasks $\mathcal{P}, \mathcal{J}$, respectively, and asks for a set of teams $\mathcal{Z}$, one team for each task such that the maximum workload of a worker is minimized  and all tasks in $\mathcal{J}$ are completely covered. 
We transform an instance of {\load} to an instance of {\balanceloadcover}, by setting $\mathcal{P}$ and $\mathcal{J}$ to be the experts and the tasks, respectively, of the {\balanceloadcover} problem. 
We now claim that for $0 < \lambda < \frac{1}{k\cardinality}$, $\mathcal{Q}$ is a solution to the {\balanceloadcover} problem if and only if it is also a solution to the {\load} problem.

To see this consider the following: If $\mathcal{Q}$ is the solution to the
{\load} problem with load $L(\mathcal{Q})$, then $\mathcal{Q}$ is also a solution for {\balanceloadcover} with load $L(\mathcal{Q})$ and incompleteness $C(\mathcal{Q})=0$; this is because
$\mathcal{Q}$ covers all skill requirements in the {\load} problem. 

Conversely, let $\mathcal{Q}$ be a solution of 
{\balanceloadcover} for $0 < \lambda < \frac{1}{k\cardinality}$. We will show that $\mathcal{Q}$ is also a solution for the {\load} problem by claiming that for any $0 < \lambda < \frac{1}{k\cardinality}$ the solution of {\balanceloadcover} always yields $C(\mathcal{Q})=0$.
In order to ensure that $C(\mathcal{Q})=0$ (all task skills are covered), any possible team assignment $\mathcal{Q}'$ should lead to $\lambda L(\mathcal{Q}') <  C(\mathcal{Q}')$. 
Intuitively, this means that adding more workload to the experts is always preferred, than leaving any of the task requirement skills unsatisfied.
This is always true if the cost of the largest possible workload, which is assigning one or more experts to all tasks ($L_{max}=k$), multiplied by our trade-off parameter $\lambda$, is less than the smallest possible incompleteness cost, which is leaving one skill of the task with the largest cardinality uncovered.
This is true for  $\lambda< \frac{1}{k\cardinality}$.
\end{proof}

\subsection{Variant of the {\balanceloadcover} problem}
A natural variant of {\balanceloadcover} is one where some skills of a task are required while others are not.
In this variant, each task $J_i$ has a set of required skills $J_i^r$ and a set of optional skills $J_i^o$, such that $J_i=J_i^r\cup J_i^o$ and 
$J_i^r\cap J_i^o=\emptyset$.  
The required skills have to be covered while the optional skills behave as before. 
This problem variant is formally defined as follows:

\begin{problem}[{\balanceloadrequiredcover}]
	\label{prob:blrsc}
	Given a set of $k$ tasks $\mathcal{J}=\{J_1,\ldots, J_k\}$, with 
$\mathcal{J}^r = \{J_1^r,\ldots ,J_k^r\}$ and $\mathcal{J}^o=\{J_1^o,\ldots , J_k^o\}$, 
a set of $n$ experts $\mathcal{P}=\{P_1,\ldots,P_n\}$ and a real non-negative value $\lambda\in \mathbb{R}^+$, find $k$ teams $\mathcal{Q}=\{Q_1,\ldots, Q_k\}$, such that $B(\mathcal{Q},\mathcal{J}^o,\lambda)$ is minimized and
$C(\mathcal{Q},\mathcal{J}^r)=0$.
\end{problem}
From the complexity viewpoint we have the following result.
\begin{thm}
\label{thm:blrc}
The {\balanceloadrequiredcover} problem is NP-hard.
\end{thm}
This is because this problem is NP-hard even for the case where all skills of all tasks are required by a reduction from the {\load} problem~\cite{anagnostopoulos2010power}.

\section{Algorithms}
\label{sec:algo}
In this section, we describe the algorithms we designed for solving the {\balanceloadcover} problem. 

\mpara  {The {\ExpertGreedy} algorithm:} 
{\ExpertGreedy} finds $\ell$ solutions (team assignments), each of 
which with a different maximum workload $\ell=1,\ldots,k=|\mathcal{J}|$ and at the end reports 
the solution with the best score. 
In order to do so, for each $\ell$ it finds for each expert $P_{i}$ the $\ell$ tasks with the 
least uncovered skills when assigning $P_{i}$ to these tasks, and it assigns $P_{i}$ to teams that correspond
to those tasks.
The algorithm reports the solution of the $\ell$ value that resulted in the smallest $B(\mathcal{Q},\mathcal{J},\lambda)$ team-assignment cost. 

The pseudocode of {\ExpertGreedy} is shown in Algorithm \ref{algo:exg}. 
We draw attention to line 5 of this pseudocode. 
Routine {\TopTasks} retrieves the indexes of the $\ell$ tasks with the 
smallest fraction of uncovered skills when expert $P_{i}$ is assigned to them.
To find these tasks we use a binary min-heap to preserve the incompleteness cost of all tasks in sorted order.
Furthermore, lines 6 and 7 perform update operations.
In particular, routine {\UpdateTeams} (line 6) assigns expert $P_{i}$ to the selected $\ell$ tasks, while 
routine {\UpdateTasks} (line 7) removes skills from the selected tasks that are covered by expert $P_{i}$ chosen
at any given round $i$.

\begin{algorithm}[t]
	\begin{flushleft}
	\textbf{Input:} Tasks\;$\mathcal{J}$=$\{J_1,\ldots,J_k\}$, Experts\;$\mathcal{P}$=$\{P_1,\ldots,P_n\}$, $\lambda$ \\
	\textbf{Output:}Teams\;$\mathcal{Q}$=$\{Q_{1},\ldots,Q_{k}\}$
	\end{flushleft}
	\begin{algorithmic}[1]
        \STATE $score \gets \infty$
        \FOR{$\ell\in \{0,\ldots,|\mathcal{J}|\}$}
        \STATE $\mathcal{J'} \gets \mathcal{J}, \mathcal{Q'} \gets \emptyset$
        	\FOR{$P_{i}\in \mathcal{P}$}
                \STATE $L_{i} \gets$ {\TopTasks}($P_{i},\mathcal{J'},\ell$) 
		  \STATE $\mathcal{Q'} \gets$ {\UpdateTeams}($P_{i},L_{i},\mathcal{Q'}$) 
		  \STATE $\mathcal{J'} \gets$ {\UpdateTasks}($P_{i},L_{i},\mathcal{J'}$) 
            \ENDFOR
            \IF{$score > B(\mathcal{Q'},\mathcal{J},\lambda)$}
            	\STATE $score \gets B(\mathcal{Q'},\mathcal{J},\lambda)$
                \STATE $\mathcal{Q} \gets \mathcal{Q'}$
            \ENDIF
        \ENDFOR
    \RETURN {$\mathcal{Q},score$}
      \end{algorithmic}
	\caption{\label{algo:exg} The {\ExpertGreedy} algorithm.}
\end{algorithm}

A natural property of {\ExpertGreedy} is that it essentially assigns the same amount of workload to every expert. 
Note, that when deciding which teams to select for an expert, for a specific $\ell$, 
the algorithm does not take into account the first part of the objective function, i.e. $\lambda L(\mathcal{Q})$, since it is equal to $\lambda\ell$ for all experts.

The runtime complexity of {\ExpertGreedy} is O($k^{2}n\log{k} + k^{2}nm$). 
For each maximum load and for each expert, the algorithm sorts the tasks in ascending order based on the number of skills not covered.

\mpara {The {\ProjectTopExpertGreedy} algorithm:} 
This algorithm also finds $\ell$ solutions (team assignments), 
each with a different maximum workload $\ell=1,\ldots,k=|\mathcal{J}|$ and then selects the solution with the smallest cost. 
However, it differs from the previous algorithm: while {\ExpertGreedy} greedily assigns tasks to experts,  
{\ProjectTopExpertGreedy} finds a set of ``good'' candidate experts for a specific task.
In particular, for each $\ell$, the algorithm computes for each task $J_{j}$ the cost of the objective value when expert $P_{i}$ is assigned to team $Q_{j}$, for all $i=1,\ldots,n$. 
The algorithm keeps these costs in a binary min-heap data structure, for running time efficiency.
After computing the costs of all experts, it removes the root of the heap and assigns the corresponding expert to team $Q_{j}$, only if her skillset overlaps with the uncovered skills of task $J_{j}$. 
If the expert is assigned to the team, then all covered skills of $J_{j}$ are removed.
This process continues until, either all skills of $J_{j}$ are covered, or the remaining skills do not overlap with any of the unassigned experts.
After creating $Q_{j}$, the algorithm checks if there are any experts whose loads are equal to $\ell$, and removes those experts from the pool. 
At the end of each loop of $\ell$, there is a team associated with every task, 
and the cost $B(\mathcal{Q},\mathcal{J},\lambda)$ is computed. 
The algorithm reports the solution with the lowest team-assignment cost. 

The pseudocode of {\ProjectTopExpertGreedy} is presented in Algorithm \ref{algo:texg}.
Routine {\TopExperts} (line 5) computes and returns the indexes of those experts whose skillsets cover the requirements of the given task and that have the smallest objective value. 
Routines {\UpdateTeams} (line 6) and {\UpdateExperts} (line 7) perform update operations, i.e., assign the selected experts to the team of the current task, and remove from the pool of experts those with load cost equal to $\ell$, respectively.

\begin{algorithm}[t]
	\begin{flushleft}
	\textbf{Input:} Tasks\;$\mathcal{J}$=$\{J_1,\ldots,J_k\}$, Experts\;$\mathcal{P}$=$\{P_1,\ldots,P_n\}$, $\lambda$ \\
	\textbf{Output:}Teams\;$\mathcal{Q}$=$\{Q_{1},\ldots,Q_{k}\}$
	\end{flushleft}
	\begin{algorithmic}[1]
        \STATE $score \gets \infty$
        \FOR{$\ell\in \{0,\ldots,|\mathcal{J}|\}$}
        \STATE $\mathcal{P'} \gets \mathcal{P}, \mathcal{Q'} \gets \emptyset$
        	\FOR{$J_{j}\in J$}
		\STATE $L_{i} \gets$ {\TopExperts}($J_{j},\mathcal{P'},\ell$) 
		\STATE $\mathcal{Q'} \gets$ {\UpdateTeams}($J_{j},L_{i},\mathcal{Q'}$) 
		\STATE $\mathcal{P'} \gets$ {\UpdateExperts}($\mathcal{P'},\mathcal{Q'}$)
            \ENDFOR
            \IF{$score > B(\mathcal{Q'},\mathcal{J},\lambda)$}
            	\STATE $score \gets B(\mathcal{Q'},\mathcal{J},\lambda)$
                \STATE $\mathcal{Q} \gets \mathcal{Q'}$
            \ENDIF
        \ENDFOR
    \RETURN {$\mathcal{Q},score$}
    \end{algorithmic}
	\caption{The {\ProjectTopExpertGreedy} algorithm. \label{algo:texg} }
\end{algorithm} 

In contrast to {\ExpertGreedy}, {\ProjectTopExpertGreedy} does not assign the same amount of workload to every expert. 
In fact, some experts might not be assigned to any team at all; this is the case when there are other experts whose skillsets overlap more with the tasks.

Another difference between {\ExpertGreedy} and {\ProjectTopExpertGreedy} is their running time. 
In particular, the running time of {\ProjectTopExpertGreedy} is O($k^{3}n^{2} + k^{3}nm + k^{2}n\log{n} + k^{2}m^{2}$).
For each $\ell$ value and for each task, the algorithm sorts the experts in ascending order, based on the cost obtained after considering each of them separately, and then traverses them in the same order to allocate a team, based on the experts' overlap with the task.
We improve the running time, by observing that the objective value computation when considering an expert for a specific task, does not require finding the total incompleteness cost of all tasks, but only how much the specific task is covered by the expert since the incompleteness cost in the other tasks remains constant for all experts that are being evaluated for that task.
Then, the runtime complexity becomes O($k^{3}n^{2} + k^{2}nm + k^{2}n\log{n} + k^{2}m^{2}$).
Finally, keeping a variable that stores the overall maximum load during an $\ell$ loop decreases the runtime complexity to O($k^{2}nm + k^{2}n\log{n} + k^{2}m^{2}$).

\mpara {The {\BestLoad} algorithm:} 
The {\BestLoad} algorithm is a natural extension of the {\LoadGreedy} 
algorithm proposed by Anagnostopoulos {\etal}~\cite{anagnostopoulos2010power} for the offline setting of the {\load} problem. Recall that in that problem the goal is given a set of tasks, find an assignment of teams to tasks so as to minimize the maximum load of the workers subject to the constraint that all skills of all tasks are covered.

The {\LoadGreedy} algorithm has two steps.
The first step solves optimally the linear programming relaxation of the ILP formulation 
of the above problem (see Theorem 2~\cite{anagnostopoulos2010power}).
This creates a fractional solution $\hat{X}$.
The second step of {\LoadGreedy} performs $R$ rounds, with $R=(\ln \frac{2T}{\delta})$, where $T=\textrm{max}\{mk,n\}$, where $m$ is the number of skills, $k$ is the number of tasks and $n$ is the number of experts.
The algorithm assigns an expert $P_j$ to the task $J_i$ with probability $\hat{X}_{ji}$, independently of other rounds and of other assignments within the same round.
If expert $P_j$ was assigned to task $J_i$ in at least one round, the algorithm adds the expert to the team $Q_i$.
The authors show that $R$ rounds are required to achieve complete coverage of the skills acquired by the tasks.

The {\BestLoad} algorithm we propose has the same first step as {\LoadGreedy}.
To take into account the weighing trade-off parameter $\lambda$ our {\BestLoad} 
modifies the second step.
In particular, notice that as the number of rounds increases, more experts are assigned to tasks, i.e., the load increases and the coverage decreases.
Therefore, for larger values of $\lambda$ (load becomes more important than coverage) running fewer rounds leads to a better solution.
Conversely, for smaller values of $\lambda$ (coverage becomes more important than load) the algorithm needs to run a number of rounds closer to $R$.
Based on this observation, {\BestLoad} accommodates for the different values of $\lambda$ 
by creating $R$ solutions; one after each assignment round.
Then, given a specific value of $\lambda$, it returns the solution that has the corresponding smallest cost.

The runtime complexity of the first step of {\BestLoad} depends on the method used to solve the LP relaxation of the algorithm. 
State-of-the-art LP solvers require running time polynomial in the number of constraints of the problem~\cite{gondzio1994computational,spielman2004smoothed}. 
For the coverage of all skills, O($nkm$) constraints are required.
The second step of the algorithm requires O($Tnk$) time.

\begin{observation}
\label{obs:performance}
The performance of {\BestLoad} is at least as good as the performance of {\LoadGreedy} for the {\balanceloadcover} problem.
\end{observation}
Clearly, Observation \ref{obs:performance} holds since one of the solutions considered by {\BestLoad} is the one returned by {\LoadGreedy}.

\mpara{Improving the running time of {\ExpertGreedy} and {\ProjectTopExpertGreedy}:} 
For any value of the trade-off parameter $\lambda$, continually adding more workload to the experts will increase the value of $B(\mathcal{Q},\mathcal{J},\lambda)$ in two cases: 
(i) when all task requirements have been covered, and (ii) when the benefit from decreasing the incompleteness cost is significantly smaller than the cost of increasing the maximum load.
In these two cases, we expect the first part of the objective function to grow, while the second part remains approximately constant. 
This observation allows us to improve the runtime complexity of {\ExpertGreedy} and {\ProjectTopExpertGreedy} by setting a maximum possible value for $\ell$, namely $\ell_{\textrm{max}}$, with $\ell_{\textrm{max}}<|\mathcal{J}|$.
The appropriate selection of $\ell_{\textrm{max}}$, is a trade-off between the running time, and the quality of the results.

\mpara{Solving the {\balanceloadrequiredcover} problem:} 
Here, we present how we can extend the above algorithms to solve the {\balanceloadrequiredcover} problem.
This extension is based on a pre-processing stage that accounts for the required skills that need to be covered.

Solving the {\balanceloadrequiredcover} problem is essentially the same as adding a \emph{preprocessing step} to the 
algorithms discussed in the previous paragraph. This preprocessing step makes sure that \emph{all} required skills
from all tasks are covered, with a relatively small maximum load among the experts.  More specifically, in this step, we deploy the {\LoadGreedy} algorithm proposed in~\cite{anagnostopoulos2010power} with inputs the set of experts $\mathcal{P}$ and the set of tasks  $\mathcal{J}^{r}$.
Then, we remove from each task those skills that are covered by the corresponding team members -- in this way we remove all required skills and some of the optional ones that are now covered.
On this new input, we run the algorithms we designed to solve {\balanceloadcover}.

The running time of the preprocessing step is dominated by the method used to solve the linear programming relaxation of the algorithm {\LoadGreedy}.
As described above the state-of-the-art LP solvers require time polynomial in the number of constraints~\cite{gondzio1994computational,spielman2004smoothed}.

\section{Experiments}
\label{sec:exp}
\newcolumntype{P}[1]{>{\centering\arraybackslash}p{#1}}
This section explores the practicality of our algorithms using data from three major online labor markets. 
Specifically, 
$(i)$ we evaluate and compare the performances of our three methods, {\ExpertGreedy}, {\ProjectTopExpertGreedy} and {\BestLoad}, to multiple baselines for the {\balanceloadcover} and {\balanceloadrequiredcover} problems, 
$(ii)$ we showcase the impact of the trade-off parameter $\lambda$ on the load and incompleteness cost of the solution,
$(iii)$ we provide a running time analysis of our algorithms.

For all our experiments we use a single process implementation of our algorithms on a 64-bit MacBook Pro with an Intel Core i7 CPU at 2.6GHz and 16 GB RAM.
We use the Gurobi optimizer \cite{optimization2014inc} for linear programming.
We make the code, the datasets and the chosen parameters available online \footnote{https://www.dropbox.com/sh/8zpsi1etvvvvj5k/AACU5GEJwUibO1P\_UtOPkf5Ha?dl=0}.
\subsection{Datasets}
We use data from the online labor marketplaces, {\texttt{freelancer.com}}, {\texttt{guru.com}}, and {\texttt{upwork.com}}. We refer to these datasets as {\FreelanceDataset}, {\GuruDataset}, and {\UpworkDataset}, respectively. 

Table \ref{tbl:datasets} exhibits statistics on the different sizes and skill properties of these datasets. 
\begin{table}
\small
	\centering
	\begin{tabular}{lP{1.2cm}lP{1.2cm}lP{1.2cm}lP{1.2cm}lP{1.2cm}|}
		\toprule
		\belowrulesepcolor{lightgray}
		\rowcolor{lightgray} Dataset & {\FreelanceDataset} & {\GuruDataset} & {\UpworkDataset}\\
		\aboverulesepcolor{lightgray}
		\midrule
		\# experts & $1212$ & $6120$  & $1500$\\
		\# tasks & $993$ & $3195$ & $3000$\\ \hline
		\# avg. skills/expert & 1.46 & 13.07 & 6.2\\
		\# avg. skills/task & 2.86 & 5.24 & 39.9\\
		\bottomrule
	\end{tabular}	
	\caption{A summary of the dataset statistics. 
		\label{tbl:datasets}}
\end{table}

In all datasets, skills acquired by experts that are never required by any task have been removed, since these are never used.
Note that, {\FreelanceDataset} (1212 experts, 993 tasks) and {\GuruDataset} (6120 experts, 3195 tasks) have more experts available than posted tasks, while the reverse is true for {\UpworkDataset} (1500 experts, 3000 tasks).
An interesting observation is that the ratio of expert skills to task skills is different in each of the three datasets.

\spara{Task skills:}
The {\FreelanceDataset} and {\GuruDataset} datasets include a random sample from a large pool of real tasks posted by users in these marketplaces.  
The {\UpworkDataset} dataset is a synthetic dataset obtained through a data-generation procedure similar to that used in the past \cite{anagnostopoulos2012online}; a small number of experts ($10\%$) is removed from the pool of experts in the dataset, and then subsets of their skills are repeatedly sampled to create tasks, by interpreting the union of their skills as task requirements.

\spara{Expert skills:}
All expert datasets used in this work are acquired from anonymized profiles of members registered in the three marketplaces. 
A profile itself includes a self-defined set of skills. 

\subsection{Baseline algorithms}
We compare the performance of our algorithms to the following baselines:

\spara{{\SetCoverGreedy}:}
A simple variation of the well-known greedy algorithm for {\setcover} \cite{vazirani2013approximation};
for each task, the algorithm iteratively assigns to each team the
expert whose skills overlap the most with the uncovered skills of the task and then removes these skills from the task.
The algorithm stops, either when all skills have been covered, or when none of the experts overlap with the remaining uncovered skills. 
The running time of this algorithm is O($knm^{2}$).

\spara{{\BestCostGreedy}:}
This is a variant of {\SetCoverGreedy} that takes into account the workload.
The difference is that instead of selecting the expert overlapping the most with the task, {\BestCostGreedy} assigns the expert that improves the objective function the most.
The algorithm stops when the cost cannot be further decreased. 
The running time of this algorithm is O($knm^{2}$). 

\spara{{\PairGreedy}:}
{\PairGreedy} is another intuitive greedy algorithm, which finds in each iteration the (task, expert) pair that improves the objective the most, and assigns the expert to the corresponding team. 
The drawback of this baseline is its runtime complexity, which is $O\left(k^{3}n\left(n + m\right)\right)$, thus prohibiting us from evaluating it on  real datasets.
As such, we do not report its performance.
Nevertheless, even when tested on smaller datasets, its performance is always outperformed by the proposed algorithms.

\subsection{Performance evaluation for {\balanceloadcover}}
\label{sec:perfeval}
This section demonstrates the performance of the proposed algorithms compared to the baselines, for the {\balanceloadcover} problem. 

In these experiments, we vary the trade-off parameter $\lambda$ to take values in $\{0,2,\ldots,10\}$.
We select this specific range of $\lambda$ values because it makes the impact of the trade-off parameter clear application-wise.
However, we also show the performance of our algorithms for values of $\lambda$ for which we showed that {\balanceloadcover} is NP-hard.
Furthermore, we set the parameter $\ell_{\textrm{max}}$ (the maximum number of $\ell$ iterations) to $80$ for all experiments, because we saw that in real applications it generally leads to reasonable solutions and runtime performances.

We present the results for {\balanceloadcover} for all three datasets in Figure~\ref{fig:costs}.
The $y$-axis represents the team-assignment cost ($B$) of each algorithm, and the $x$-axis 
corresponds to the value of the trade-off parameter $\lambda$.
Smaller values of cost correspond to a better solution.

We observe that the performances of our algorithms and the baselines follow a similar trend, which is consistent among the different datasets.
Furthermore, the baseline algorithms are clearly outperformed by our proposed approaches, with {\SetCoverGreedy} performing the worst --
the only exception is for $\lambda=0$, i.e., the case that completely ignores the load of the experts.
This is because {\SetCoverGreedy} always returns a team that covers all the task requirements and ignores the load of the experts.
{\BestCostGreedy} is also outperformed by our proposed algorithms.
Note that for $\lambda=0$, {\BestCostGreedy} is able to find solutions where no requirement is left uncovered.
However, as $\lambda$ increases, the algorithm continues covering all of the task requirements without decreasing the workload, which leads to the linear increase of the total cost.
The only exception is for the dataset {\UpworkDataset}, Figure \ref{fig:costs}(c), where for $\lambda=4$ the algorithm begins compromising incompleteness cost for less workload, but the total cost remains significantly larger compared to the other algorithms.
The performance of {\BestCostGreedy} is followed by {\BestLoad}.
One observation is that {\BestLoad} performed significantly better than the original algorithm {\LoadGreedy}.
We do not present these results because as explained in Observation \ref{obs:performance} {\BestLoad} always performs at least as good as {\LoadGreedy}.
Recall that {\LoadGreedy} finds a single solution that optimizes the workload, while it covers all the task requirements and its final solution is completely independent of the trade-off parameter $\lambda$.
Therefore, since the load of the solution is constant and the coverage is 0 the team-assignment cost increases linearly with the coefficient $\lambda$.

Now, we illustrate the performance of our algorithms for values of $\lambda$ for which we showed that the {\balanceloadcover} problem is NP-hard.
The corresponding performances for the three datasets can be seen in Figure~\ref{fig:costs} as subplots.
As expected, the closer $\lambda$ is to 0 the closer the algorithmic performances are, but as $\lambda$ increases the difference in the performance of our proposed algorithms and the baselines also increases.
Overall, we observe that our algorithms, namely {\ExpertGreedy}, {\ProjectTopExpertGreedy} and {\BestLoad}, outperform the baseline algorithms and we discuss their individual trade-off and efficiency differences below.

\subsection{Performance evaluation for {\balanceloadrequiredcover}}
We perform another set of experiments to demonstrate the performances of the algorithms for the {\balanceloadrequiredcover} problem.
In these experiments, we vary the fraction of required skills in the tasks as follows; with probability $p_s$ we independently define each skill of every task to be a required skill otherwise it is considered optional.
In Figure \ref{fig:probs}, we study how the algorithms and baselines perform for a range of $p_s$ values and a fixed $\lambda=4$.
We see that the observations for the algorithmic comparisons are similar to the ones made in the previous experiment for the {\balanceloadcover} problem.
Note that for $p_s$=0 no skill is required and therefore the algorithms perform exactly as in the {\balanceloadcover} problem, while for $p_s$=1 all skills are required and the performance of all algorithms is the same and equal to the result of the pre-processing stage.
  
\begin{figure*}[t!]
 \centering
\begin{minipage}[t]{0.32\linewidth}
\centering
    \includegraphics[width=\linewidth]{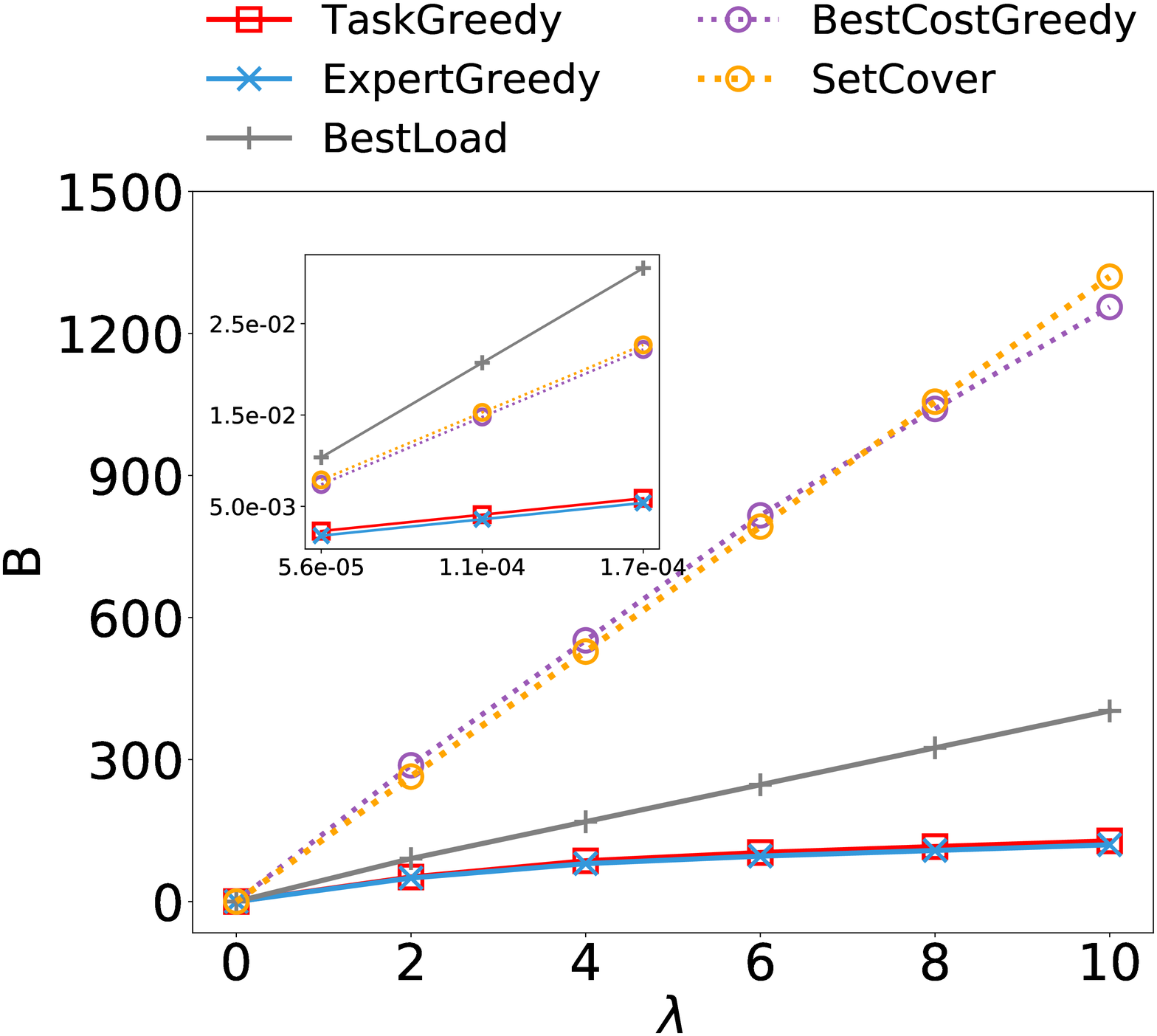}
    \small{(a)}
\end{minipage}%
    \hfill%
\begin{minipage}[t]{0.32\linewidth}
\centering
    \includegraphics[width=\linewidth]{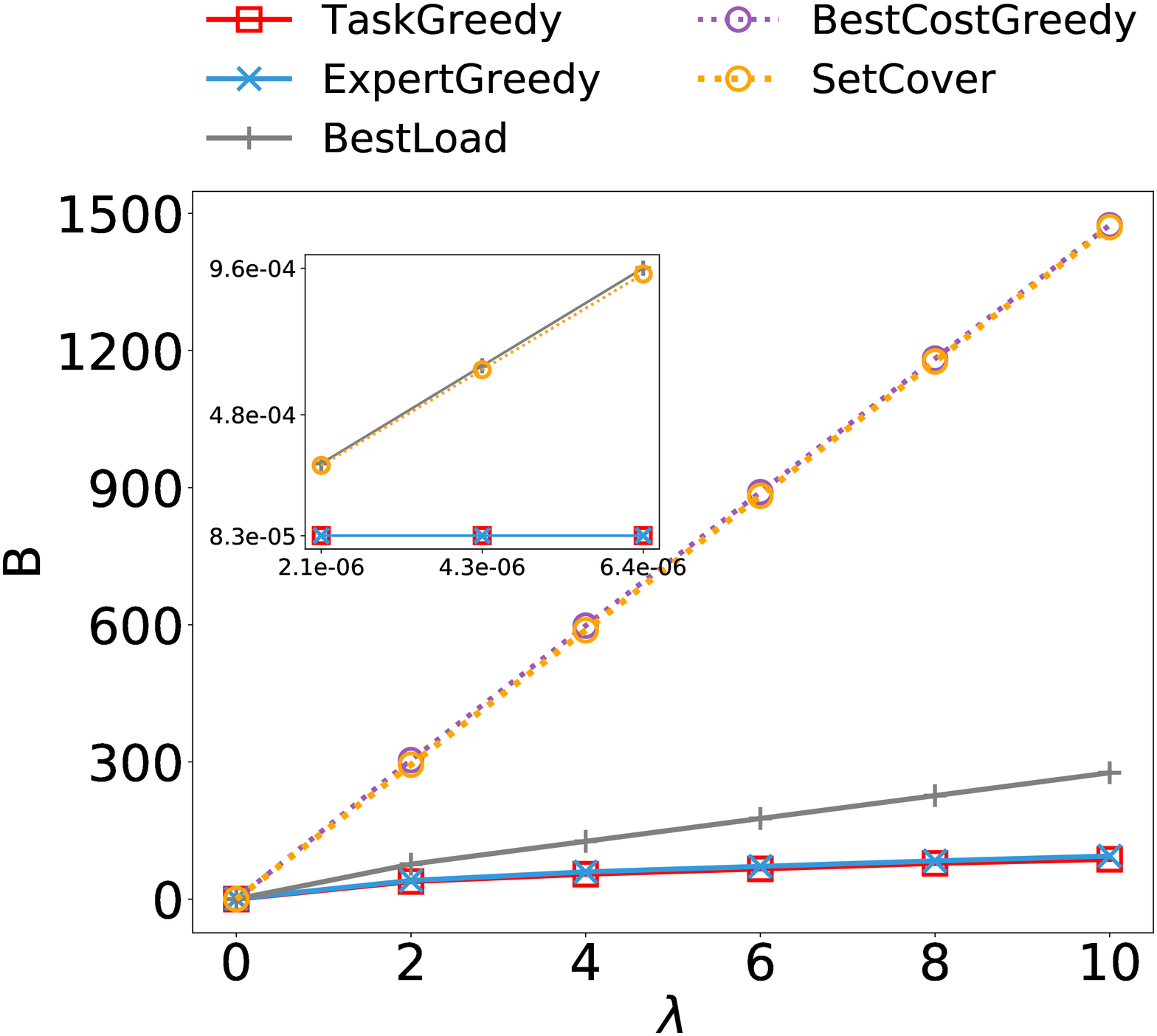}
    \small{(b)}
\end{minipage} 
\hfill%
\begin{minipage}[t]{0.32\linewidth}
\centering
    \includegraphics[width=\linewidth]{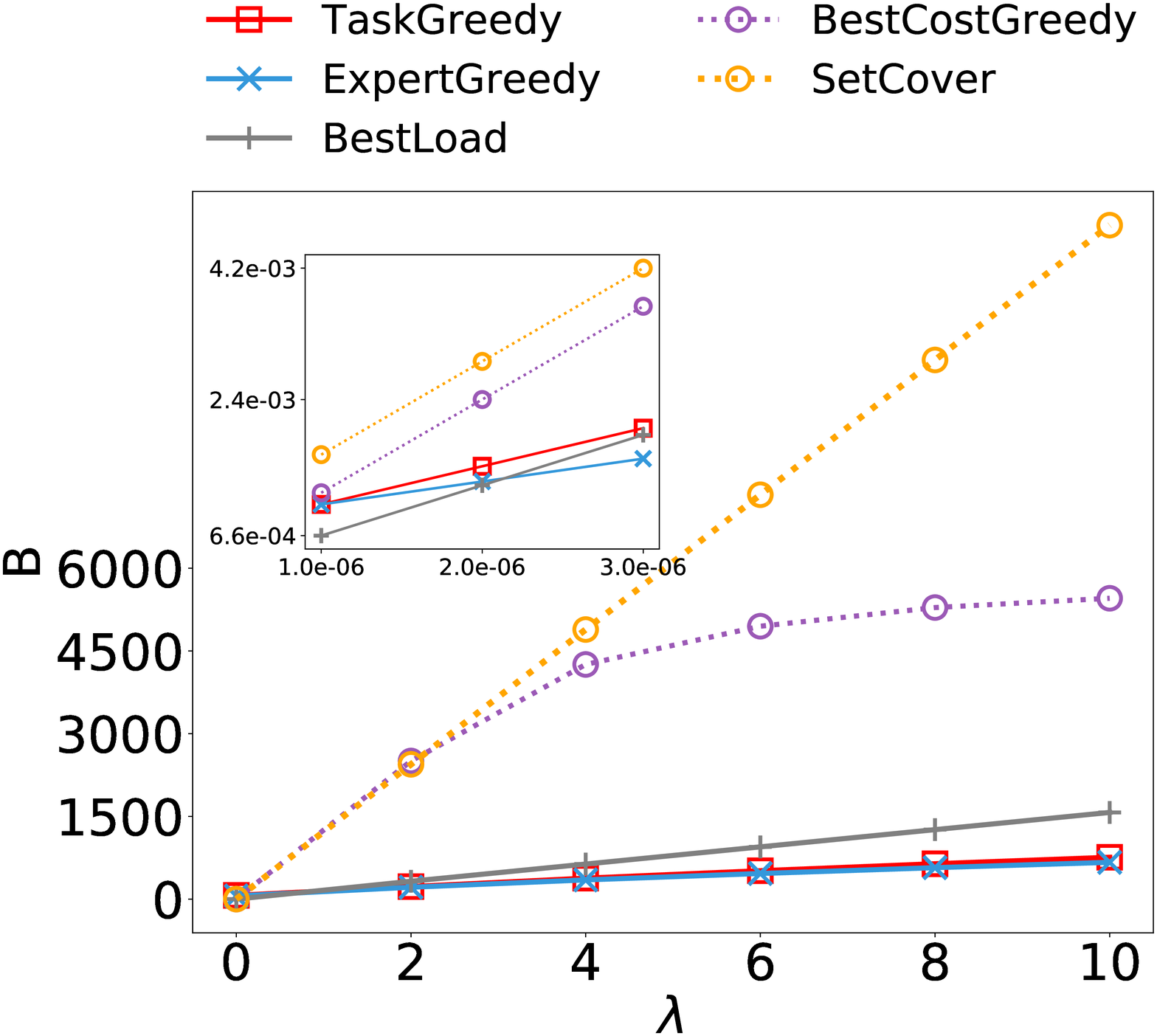}
    \small{(c)}
\end{minipage}
 \caption{Team-assignment cost ($B$) of algorithms and baselines for values of $\lambda=\{0,2,\ldots,10\}$ and $\ell_{\textrm{max}}=80$. The subplots correspond to values of $\lambda$ for which {\balanceloadcover} is NP-hard. Columns correspond to different datasets: (a) {\FreelanceDataset}; (b) {\GuruDataset}; (c) {\UpworkDataset}}.
 \label{fig:costs}
 
\end{figure*}

\begin{figure*}[t!]
 \centering
\begin{minipage}[t]{0.33\linewidth}
\centering
    \includegraphics[width=\linewidth]{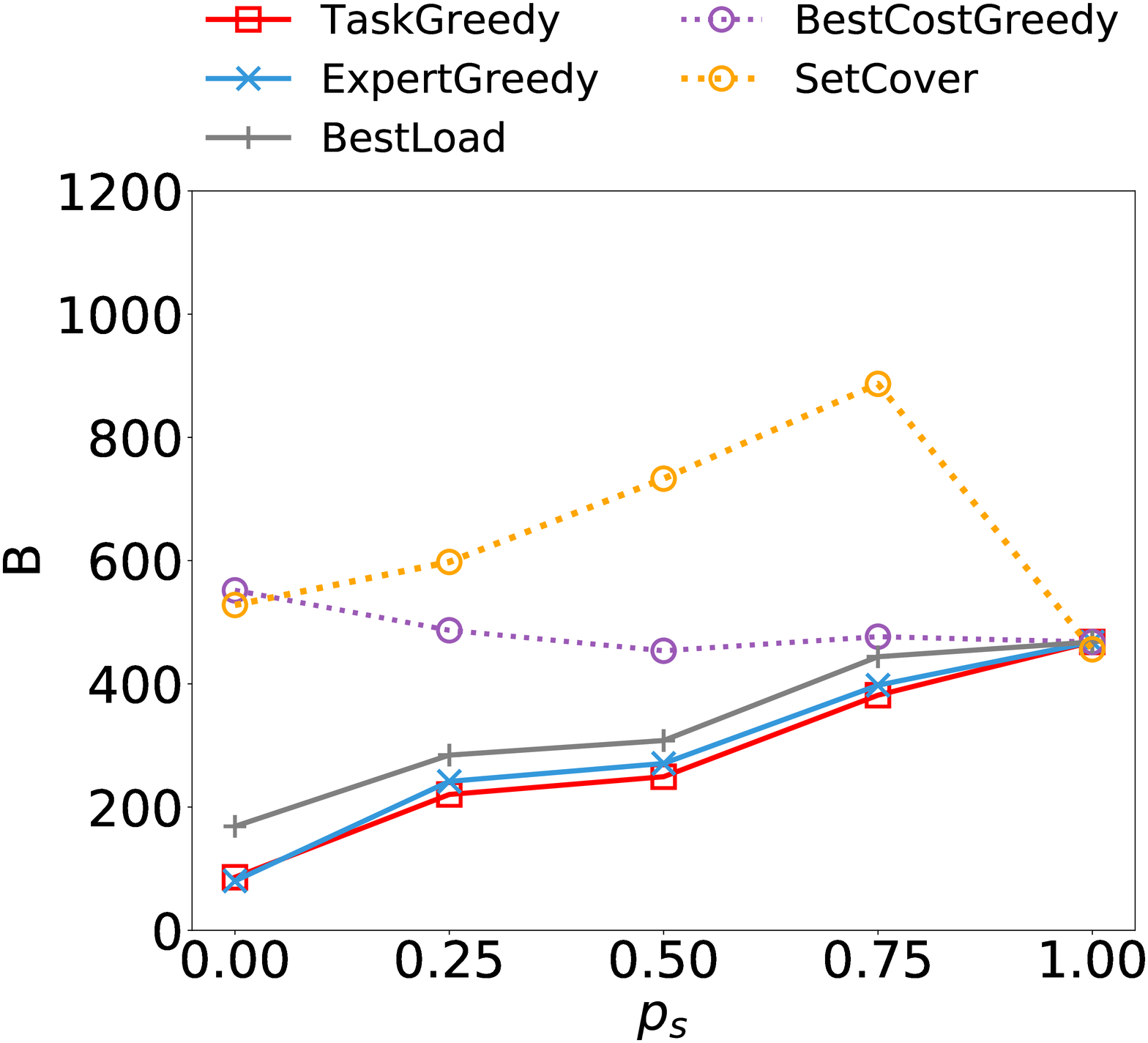}
    \small{(a)}
\end{minipage}%
    \hfill%
\begin{minipage}[t]{0.33\linewidth}
\centering
    \includegraphics[width=\linewidth]{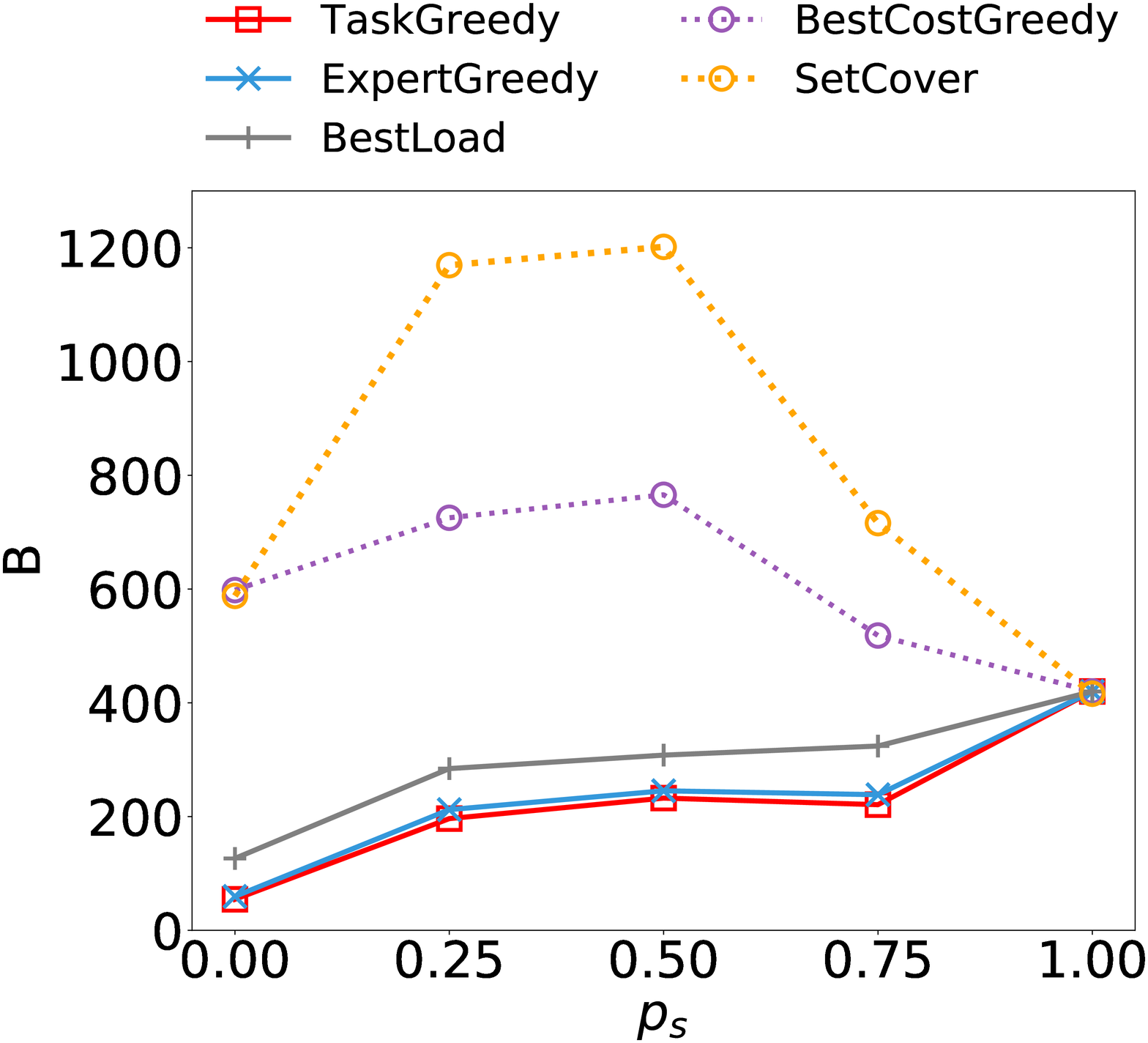}
    \small{(b)}
\end{minipage} 
\hfill%
\begin{minipage}[t]{0.33\linewidth}
\centering
    \includegraphics[width=\linewidth]{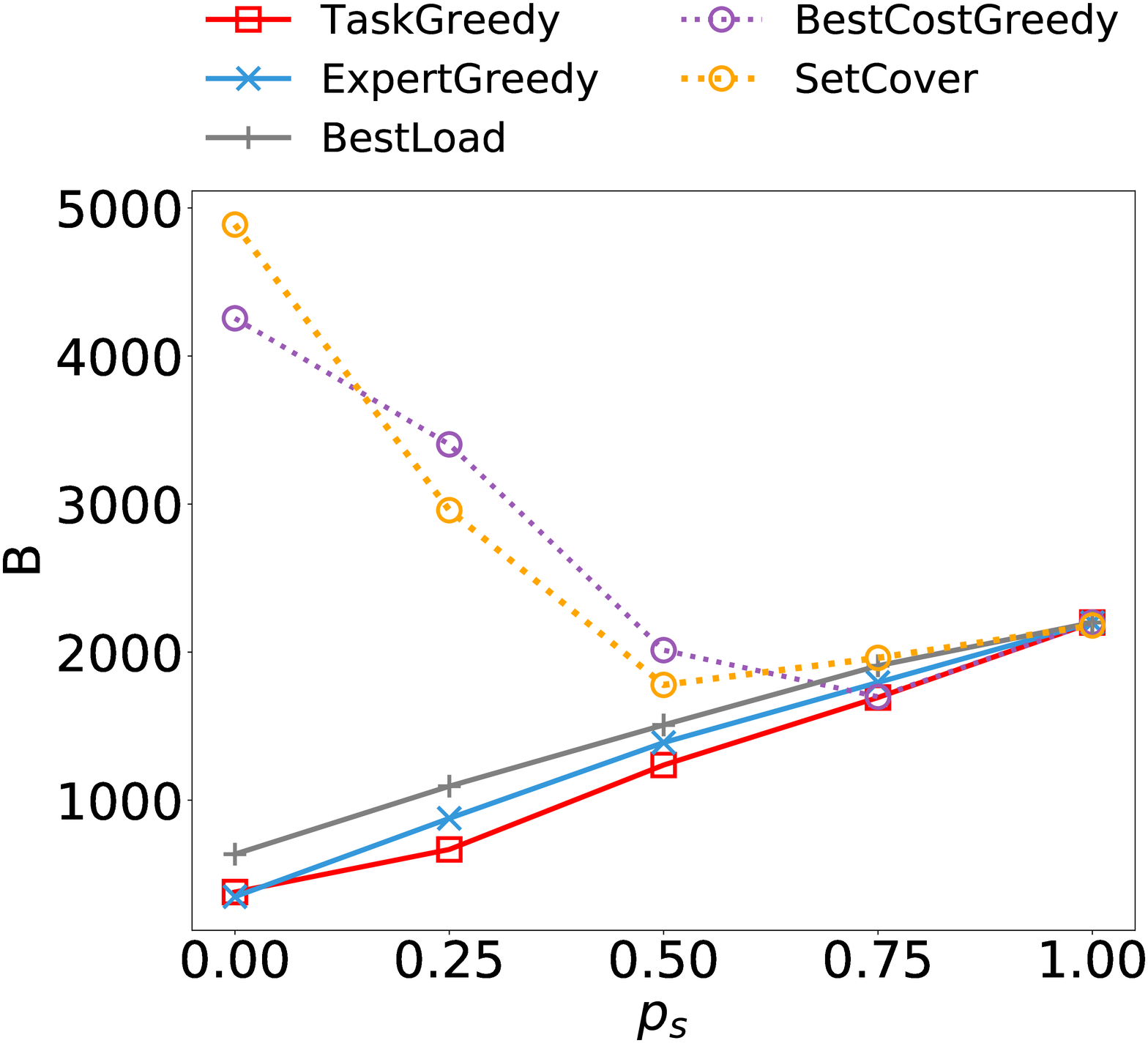}
    \small{(c)}
\end{minipage}
 \caption{Team-assignment cost ($B$) of algorithms and baselines for values of $p_s$=$\{0,0.25,0.5,0.75,1\}$, $\lambda = 4$ and $\ell_{\textrm{max}}=80$. Columns correspond to different datasets: (a) {\FreelanceDataset}; (b) {\GuruDataset}; (c) {\UpworkDataset}}
 \label{fig:probs} 
\end{figure*}

\subsection{Trade-off parameter $\lambda$}
We study the behavior of our proposed algorithms for different load and incompleteness cost trade-off values.
We begin by setting $\lambda=0$, i.e., we ignore the workload and ensure complete coverage (incompleteness cost is 0), and increase $\lambda$ to observe how the trade-off between load and incompleteness cost changes.
The results are shown in Figure~\ref{fig:tradeoffs}.
The $y$-axis shows the load cost, and the $x$-axis the incompleteness cost for the specific load.

As expected, for $\lambda$ close to $0$, our algorithms yield solutions with low incompleteness cost and high workload, while increasing $\lambda$ changes this balance accordingly.
Note that {\BestLoad} lacks trade-off capabilities, compared to {\ExpertGreedy} and {\ProjectTopExpertGreedy}.
This is because the first step of the algorithm, which creates the optimal fractional solution for the {\load} problem \cite{anagnostopoulos2010power}, is oblivious
to the parameter $\lambda$.
Thus, even though the second step of the algorithm weighs the trade-off parameter $\lambda$ the trade-off capabilities are restricted by the assignment probabilities created in the first step.
Therefore, what we see in Figure~\ref{fig:tradeoffs} is that for our datasets and the examined range of $\lambda$ the values of load and incompleteness achieved by the algorithm are the same except from the solution for $\lambda=0$.
A quality of {\ExpertGreedy} and {\ProjectTopExpertGreedy} is that improving the cost in one of the two components is achieved by paying a moderate price for the other component.
For instance, assume a customer using {\texttt{guru.com}} that considers task requirements to be important.
By observing Figure \ref{fig:tradeoffs}(a) we can set $\lambda=2$ to create teams that would satisfy both,
the customer and the experts, as for a reasonable maximum load $\sim20$ {\ExpertGreedy} and {\ProjectTopExpertGreedy} induce very small incompleteness cost $\sim5$.
Now, if another customer prefers hiring few people at the cost of incompleteness, we can set $\lambda=4$ to achieve load $\sim15$ for an incompleteness cost of $\sim30$, thus weighing differently the two components, yet always reasonably.
\begin{figure*}[htbp]
 \centering
\begin{minipage}[t]{0.32\linewidth}
\centering
    \includegraphics[width=\linewidth]{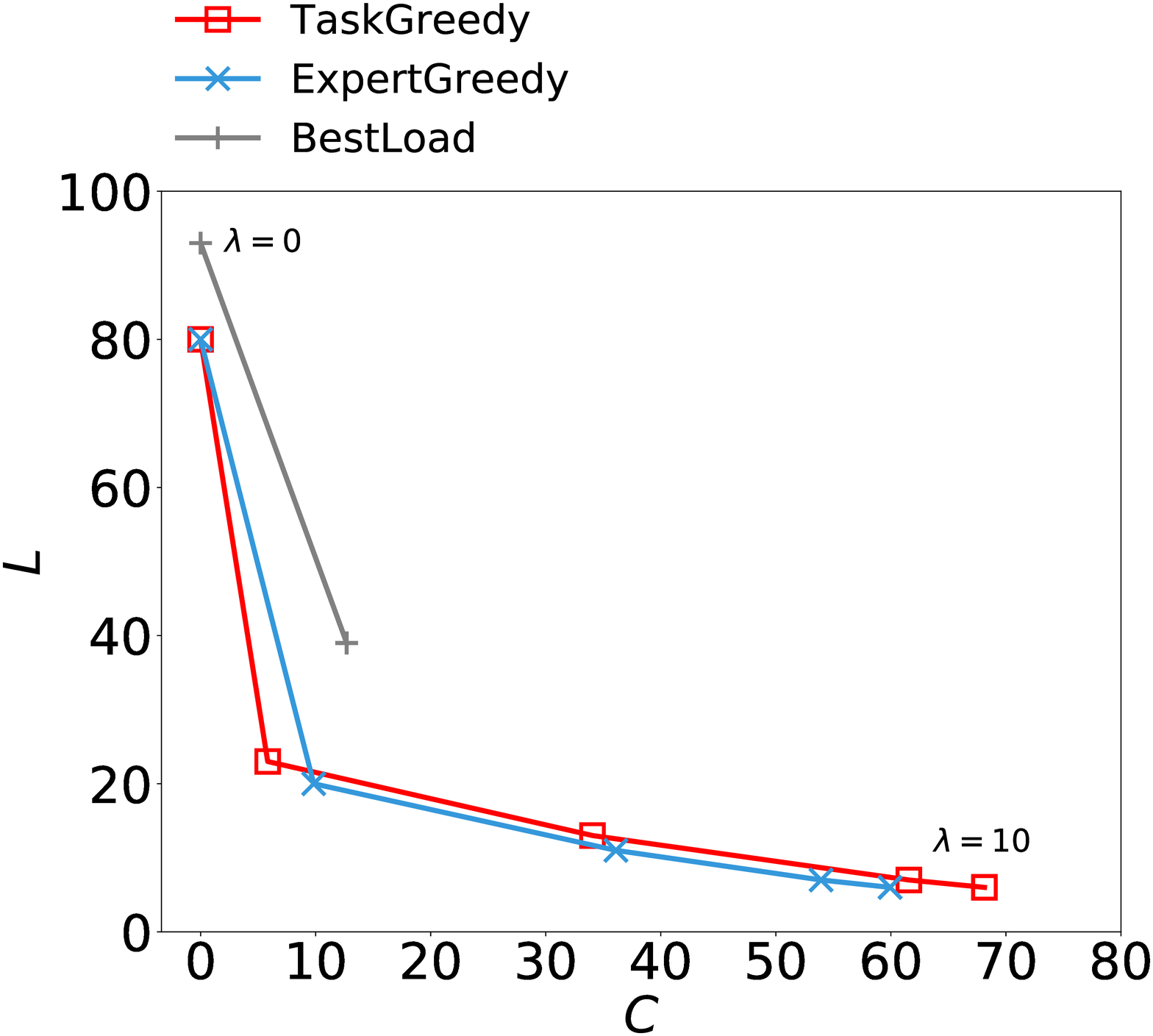}
    \small{(a)}
    \label{fig:tradeoff_fr}
\end{minipage}%
    \hfill%
\begin{minipage}[t]{0.32\linewidth}
\centering
    \includegraphics[width=\linewidth]{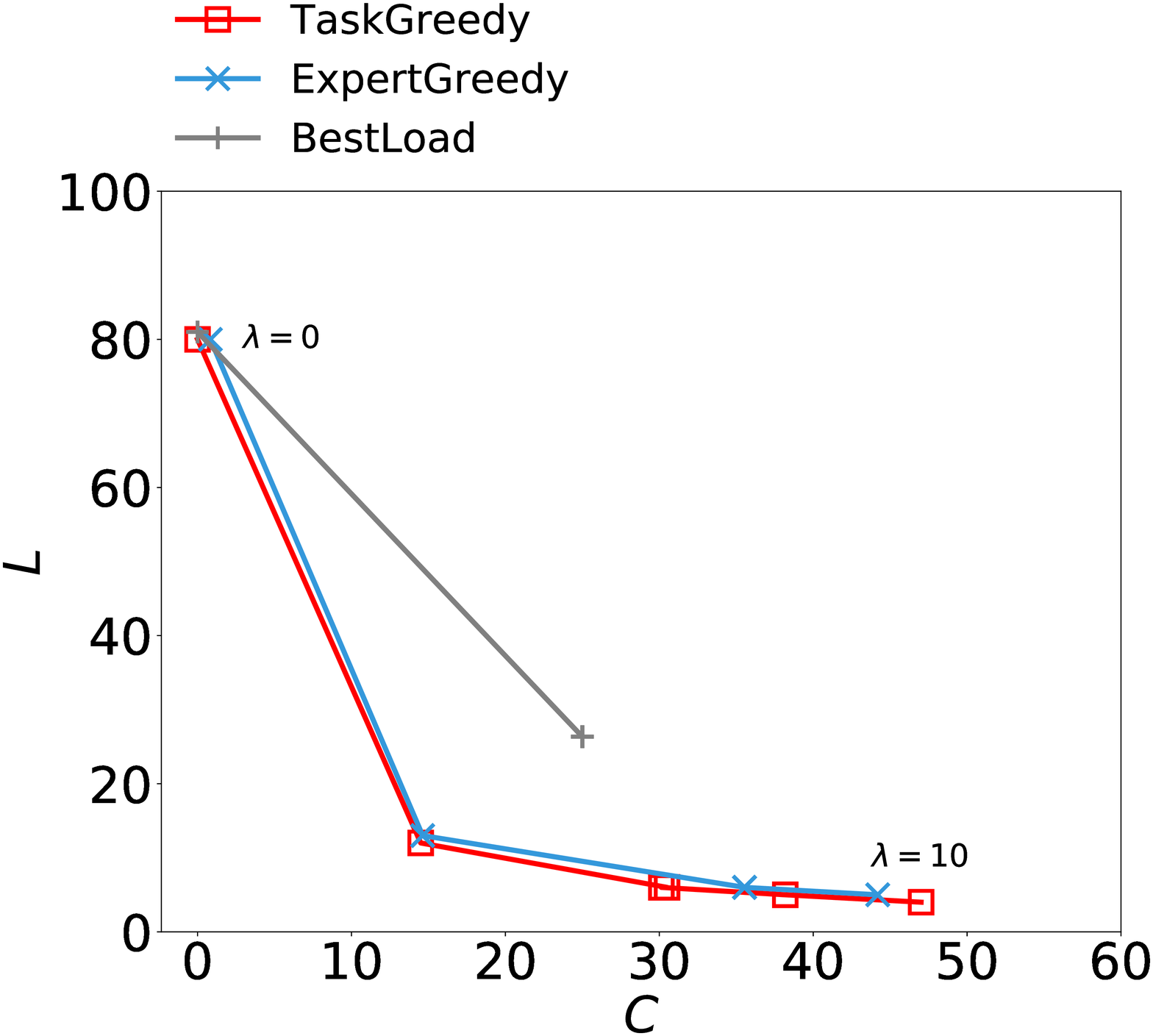}
    \small{(b)}
    \label{fig:tradeoff_gr}
\end{minipage} 
\hfill%
\begin{minipage}[t]{0.32\linewidth}
\centering
    \includegraphics[width=\linewidth]{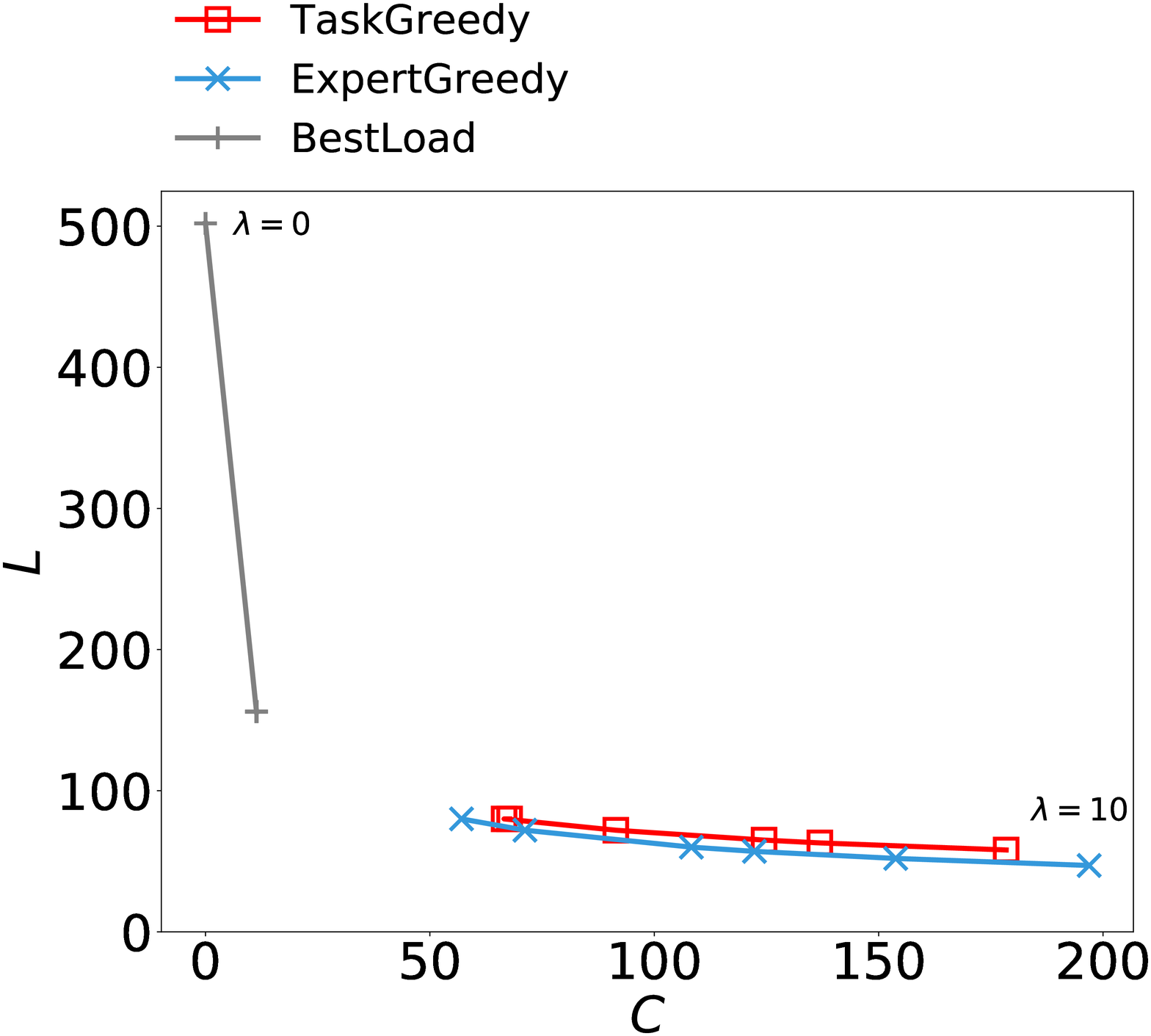}
    \small{(c)}
    \label{fig:tradeoff_up}
\end{minipage} 
\caption{Trade-off between load cost ($L$) and incompleteness cost ($C$) for $\ell_{\textrm{max}}=80$. The labels next to the first and last data points correspond to $\lambda=0$ and $\lambda=10$, respectively. The in-between points shown in the curve correspond to $\lambda=\{2,4,6,8\}$. Columns correspond to different datasets: (a) {\FreelanceDataset}; (b) {\GuruDataset}; (c) {\UpworkDataset}}
 \label{fig:tradeoffs}
\end{figure*}

Note, that the baseline algorithms are omitted from this plot. 
This is because {\BestCostGreedy}  always maintains the incompleteness cost very low, which requires workload that is much larger than the ones induced by the proposed algorithms.  
On the other hand, {\SetCoverGreedy} lacks trade-off capabilities, since its solution is always independent of $\lambda$ with 0 incompleteness cost and constant load.

Figure~\ref{fig:tradeoffs} allows us to further investigate the properties of our algorithms for {\balanceloadcover}.
A first observation is that all algorithms demonstrate a smooth transition on the load and incompleteness cost as the trade-off parameter changes. 
Note in Figures \ref{fig:tradeoffs}(a) and \ref{fig:tradeoffs}(b) that assigning a maximum workload of $80$ is enough to achieve complete task coverage.
In fact, for the same datasets, even if the load decreases to $\sim20$, the incompleteness cost increases by little.
However, this is not the case for Figure \ref{fig:tradeoffs}(c) (\UpworkDataset), where it is clear that the load should be more than $80$ ($\ell_{\textrm{max}}$) to reach complete coverage (0 incompleteness cost).
This occurs because, in the specific dataset, both the experts and the average number of skills acquired by them are significantly less than the tasks and the skills required by the tasks, respectively, which requires creating large teams and utilizing the same experts many times to achieve full coverage.
Even the baseline {\LoadGreedy} that guarantees complete coverage with minimum workload cost, needs minimum load 548 for the specific dataset to accomplish full coverage.

To showcase the differences of 
the algorithms as depicted in this experiment, we compare {\ProjectTopExpertGreedy} with  {\ExpertGreedy} and {\BestLoad}, for the {\UpworkDataset} dataset (Figure \ref{fig:tradeoffs}(c)).
Recall that the {\ProjectTopExpertGreedy} algorithm assigns experts to tasks, based on how suitable they are for the task individually, and not within a team.
Therefore, for datasets such as {\UpworkDataset}, where there are fewer experts and expert skills compared to tasks and task requirements, the algorithm becomes less effective as it cannot evaluate a newly-added person is the best option for the whole team.
However, in a dataset such as {\GuruDataset} (Figure \ref{fig:tradeoffs}(b)) where the experts acquire on average more and a larger variety of skills than the tasks require, we observe that {\ProjectTopExpertGreedy} performs slightly better, or the same compared to the other two algorithms.
This is because the skill ``surplus'', leaves to the algorithm
room for seemingly wrong local choices, as it will be able to compensate for them by using the skills of some other of the remaining experts.

\subsection{Runtime analysis}
Finally, we investigate the running time efficiency of our algorithms. 
Figure~\ref{fig:time} shows the average running times for all algorithms and datasets when setting the parameter $\lambda=4$.
The running time complexities of the algorithms are independent of $\lambda$ so its selection does not affect the running time results.
The times are averaged over $5$ runs for the {\balanceloadcover} problem -- the results for {\balanceloadrequiredcover} are similar and omitted.

\begin{figure}[t!]
 \centering
    \includegraphics[width=0.35\linewidth]{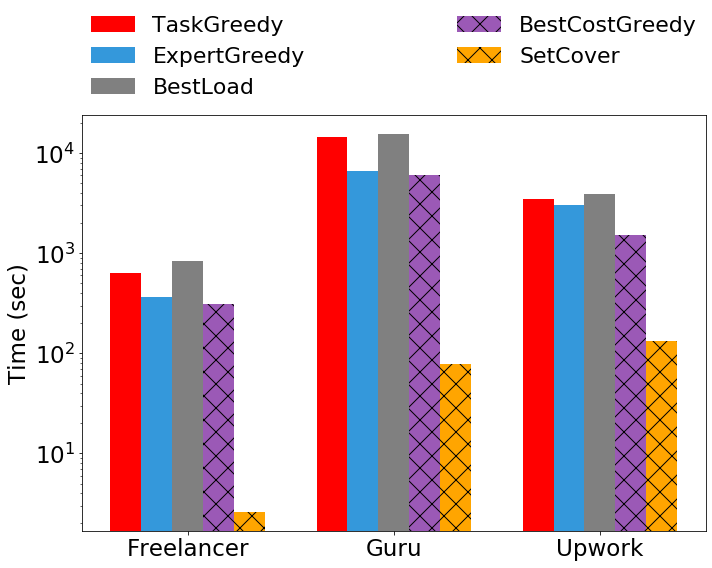}
    \label{fig:time_fr}
\caption{Average running time (sec) of algorithms and baselines over 5 runs, in logarithmic scale, for $\lambda = 4$, $\ell_{\textrm{max}}=80$. 
The three bar charts correspond to the datasets: {\FreelanceDataset}, {\GuruDataset}, {\UpworkDataset}}.
 \label{fig:time}
\end{figure}

We use the baselines {\SetCoverGreedy} and {\BestCostGreedy} as indicators of how well our algorithms perform in terms of running time because they have the best runtime complexity.
Even though their asymptotic complexity is the same, {\BestCostGreedy} is slower than {\SetCoverGreedy}.
This is because the two algorithms have different stopping criteria, the former depending on the improvement of the team-assignment cost, and the latter on the coverage of the skills.
Note that simply comparing the asymptotic running times of the different algorithms (see Section \ref{sec:algo}) is not sufficient.
In fact, there are multiple factors we need to consider, such as constants, dominating factors that depend on the properties of the datasets, efficient implementations, etc.

In Figure \ref{fig:time} datasets {\FreelanceDataset} and {\GuruDataset} show that {\ExpertGreedy} is much faster than {\ProjectTopExpertGreedy} and {\BestLoad} for datasets $k < n$ (the $y$-axis is in logarithmic scale).
Yet, for the dataset {\UpworkDataset}, where $n < k$, we see that even though {\ExpertGreedy} remains the fastest algorithm, the running time of {\ProjectTopExpertGreedy} is also very close.
One possible explanation is that having fewer experts than tasks, with fewer skills on average allows {\ProjectTopExpertGreedy} to find teams faster in this dataset; yet {\ProjectTopExpertGreedy} is consistently slower than {\ExpertGreedy} for all datasets.
Thus, we can  conclude that overall {\ExpertGreedy} is the most efficient of our algorithms.

\section{Conclusion}
\label {sec:concl}
In this paper, we introduced {\balanceloadcover}, a team-formation problem where given a collection of tasks and a pool of experts, the goal is to form teams such that each team is associated with a task and it covers it as well as possible, while at the same time, the maximum workload of the chosen experts is also minimized.
We also considered a variant of this problem where each task has some set of required skills that are required to be covered by the formed teams.
To the best of our knowledge, we are the first to combine the coverage of tasks and the workload of experts into a single objective.
We showed that our problems are NP-hard and designed efficient heuristics for solving them. 
Our experiments with three real-world datasets from online labor markets demonstrate the efficiency and the efficacy of our algorithms, and their superiority compared to other heuristics. 

\clearpage
\bibliographystyle{abbrv}
\bibliography{references}

\end{document}